\documentclass[letterpaper, 10 pt, conference]{ieeeconf}  

\IEEEoverridecommandlockouts                              
\usepackage[utf8]{inputenc}
\usepackage[T1]{fontenc}
\usepackage{xcolor}
\usepackage{graphicx}
\usepackage{flushend}

\usepackage{algorithm}
\usepackage{algpseudocode}
\algrenewcommand\algorithmicrequire{\textbf{Input:}}
\newcommand{\CommentState}[1]{\Statex\hspace{\algorithmicindent}{\color{blue}// #1}}

\usepackage{amsmath,amsfonts,amssymb}
\usepackage{bm,fixmath} 
\usepackage{accents}

\newtheorem{remark}{Remark}

\newtheorem{proposition}{Proposition}

\newtheorem{assumption}{Assumption}
\newtheorem{lemma}{Lemma}

\newcommand{\ubar}[1]{\underaccent{\bar}{#1}}

\DeclareMathOperator*{\argmin}{arg\,min}
\DeclareMathOperator{\proj}{proj}

\newcommand{\norm}[1]{\left\lVert#1\right\rVert}
\newcommand{\ev}[1]{\mathbb{E}\left[#1\right]}

\newcommand{\N}{\mathbb{N}}
\newcommand{\R}{\mathbb{R}}

\newcommand{\C}{\mathcal{C}}

\newcommand{\e}{\mathbold{e}}
\newcommand{\x}{\mathbold{x}}
\newcommand{\y}{\mathbold{y}}

\newcommand{\1}{\pmb{1}}

\def\lmin{{\ubar{\lambda}}}
\def\lmax{{\bar{\lambda}}}

\title{\LARGE \bf
Online Distributed Learning with Quantized Finite-Time Coordination}

\author{Nicola Bastianello, Apostolos I. Rikos, Karl H. Johansson 
\thanks{This work was partially supported by the European Union’s Horizon 2020 research and innovation programme under grant agreement No. 101070162, and partially by Swedish Research Council Distinguished Professor Grant 2017-01078 Knut and Alice Wallenberg Foundation Wallenberg Scholar Grant.}
\thanks{N. Bastianello and K. H. Johansson are with the School of Electrical Engineering and Computer Science and Digital Futures, KTH Royal Institute of Technology, Sweden,
        {\tt\small \{ nicolba | kallej\}@kth.se}.}%
\thanks{Apostolos~I.~Rikos is with the Department of Electrical and Computer Engineering, Division of Systems Engineering, Boston University, Boston, MA 02215, US. E-mail: {\tt arikos@bu.edu}.}
}

\begin{document}

\maketitle
\thispagestyle{plain}
\pagestyle{plain}

\begin{abstract}
In this paper we consider online distributed learning problems. 
Online distributed learning refers to the process of training learning models on distributed data sources. 
In our setting a set of agents need to cooperatively train a learning model from streaming data.
Differently from federated learning, the proposed approach does not rely on a central server but only on peer-to-peer communications among the agents. This approach is often used in scenarios where data cannot be moved to a centralized location due to privacy, security, or cost reasons. In order to overcome the absence of a central server, we propose a distributed algorithm that relies on a \textit{quantized, finite-time coordination} protocol to aggregate the locally trained models. 
Furthermore, our algorithm allows for the use of \textit{stochastic gradients} during local training. Stochastic gradients are computed using a randomly sampled subset of the local training data, which makes the proposed algorithm more efficient and scalable than traditional gradient descent. 
In our paper, we analyze the performance of the proposed algorithm in terms of the mean distance from the online solution. Finally, we present numerical results for a logistic regression task.
\end{abstract}

\section{Introduction}\label{sec:intro}

Recent technological advances have led to the widespread implementation of multi-agent systems in several applications, ranging from power grids to robotics, from traffic to sensor networks, to name a few \cite{molzahn_survey_2017,nedic_distributed_2018}. 
These systems, composed of interconnected agents equipped with computational and communication resources, enable the collection of data on many different phenomena at an unprecedented level of detail \cite{li_federated_2020}. 
Processing these data in order to train machine learning models has therefore become a research objective of central importance, around which \textit{federated learning} (FL) was born \cite{li_federated_2020,gafni_federated_2022}.

Differently from traditional machine learning approaches in which data are collected and processed at a single location, the goal of FL is to enable a cooperative learning process that does not require the agents to directly share private data. 
The usual architecture of federated learning set-ups includes a set of agents that store and locally process data, and a \textit{fusion center} that communicates with them and aggregates the results of their computations \cite{mcmahan_communication-efficient_2017}. However, in many applications it may not be possible (or desirable) to implement such an architecture. 
Indeed, the fusion center is a single point of failure, and (temporarily) removing it halts the learning process entirely. 
For this reason, researchers have recently focused on developing fully distributed solutions. 
In these solutions agents aggregate their results without needing a fusion center \cite{alghunaim2022unified}. 
This paper follows the same line of research and focuses on designing \textit{distributed} learning algorithms. 
In these algorithms agents employ peer-to-peer communications to share and aggregate the results of local processing. 
This set-up builds the foundation of this paper, and is commonly assumed in \textit{distributed optimization}, \cite{nedic_distributed_2018,yang_survey_2019}. 

Relying on peer-to-peer communications instead of a fusion center introduces many advantages for FL algorithms (e.g., enhances robustness) but poses bigger design challenges. 
Indeed, a fusion center enables dissemination of information (\textit{e.g.} local computations) to all agents in a single communication round. 
On the contrary, peer-to-peer communications require multiple communication rounds. 
Several classes of distributed algorithms have been proposed to overcome this challenge. 
In particular, distributed (sub-)gradient methods, gradient tracking methods, and primal-dual methods \cite{notarstefano_distributed_2019}. 
While the gradient-based approaches usually rely on average consensus as a coordination and aggregation technique, in this paper we explore the use of \textit{finite-time coordination} (FTC) \cite{2023:Rikos_Johan_IFAC} as a peer-to-peer substitute for the fusion center.
%
In particular, the proposed algorithm employs FTC to aggregate local gradient descent steps (as computed by the agents on the data they store) through multiple rounds of communications. 
This approach is similar to Near-DGD \cite{berahas_balancing_2019} (which however employs average consensus).


Decentralized learning algorithms (either relying on a fusion center or peer-to-peer coordination among agents) face the major challenge of operating over networks with limited communication constraints. 
Specifically, in practice communication between agents is done via bandwidth limited channels. 
For example, the usage of wireless networks \cite{qian_distributed_2022}, or the training high dimensional models, would significantly increase the communication burden \cite{richtarik_3pc_2022}. 
Therefore, guaranteeing efficient communication among agents is one of the main challenges for FL algorithms \cite{gafni_federated_2022}. 
One of the approaches for guaranteeing efficient communication is via \textit{quantization}. 
The basic idea of quantization is to represent the numerical values used in the model's parameters using fewer bits than their full precision, while still preserving a reasonable level of accuracy. 
Addressing the constraint of limited communication by utilizing quantized communication among agents is built into our proposed algorithm.
Specifically, the finite-time coordination protocol only employs quantized communication. 
We remark that, differently from distributed algorithms that use average consensus-based aggregation \cite{pu_robust_2020} (as also Near-DGD does), the proposed algorithm is robust to quantized communications, reaching an approximate solution to the learning problem. 
This approximation depends on the utilized quantization level. 
Moreover, our algorithm can be deployed on \textit{directed networks}. 
This deployment usually requires specialized reformulation of gradient-based algorithms \cite{xi_add-opt_2018,xin_decentralized_2020}.

Combining the aforementioned characteristics, in our paper we present a distributed learning algorithm. 
During the operation of our algorithm agents coordinate by exchanging quantized messages via a finite-time coordination protocol.  
We also present a convergence analysis for the proposed algorithm which characterizes the effect of quantization on the accuracy of the learned model, providing a guide to tune the quantization level. 
Additionally, our convergence analysis accounts for the effect of \textit{stochastic gradients}. 
Stochastic gradients are inexact gradients computed by the agents using only a portion of the local data \cite{xin_decentralized_2020}. 
This approach reduces the computational burden and speeds up the training process. 
In this paper we also characterize how stochastic gradients degrade the accuracy of the proposed method.
Finally, our analysis in this paper also accounts for time-variability of the local costs. 
Time-variability of local costs refers to the fact that the agents' local cost functions associated with different training examples in a dataset can vary over time. 
Indeed, in many applications, as agents continuously collect new data, local cost functions may vary over time. 
Thus, the goal of every agent in the network becomes that of solving an \textit{online learning} problem \cite{dallanese_optimization_2020,simonetto_time-varying_2020,yuan_can_2020}.


The main contributions of our paper are the following: 
\begin{enumerate}
    \item We present a distributed projected gradient-based learning algorithm. 
    Differently from federated learning, the agents do not rely on a central coordinator. Coordination is achieved by a quantized finite-time coordination protocol. 
    Our proposed algorithm employs quantized communication between agents and accounts for the effect of stochastic gradients (see Algorithm~\ref{alg:main-algorithm}). 
    \item We analyze our algorithm's convergence to the optimal solution. 
    Specifically, in our analysis we illustrate the effect on the performance that quantized communications, stochastic gradient, and time-varying costs have (see Proposition~\ref{pr:mean-convergence}). 
\end{enumerate}

\paragraph*{Outline}
In section~\ref{sec:prob_form} we formulate the problem and describe the set-up of interest. Section~\ref{sec:algorithm} presents and discussed the proposed algorithm, whose convergence is analyzed in section~\ref{sec:convergence}. Section~\ref{sec:simulations} presents numerical results for a logistic regression task.

\paragraph*{Notation}
We denote with $\mathbb{N}$ and  $\mathbb{R}$ the set of integer and real numbers, respectively.
For any $a \in \mathbb{R}$, the greatest integer less than or equal to $a$ is denoted $\lfloor a \rfloor$, while the smallest integer greater than or equal to $a$ is denoted as $\lceil a \rceil$.
We denote a directed graph (digraph) of $N$ agents by $\mathcal{G} = (\mathcal{V}, \mathcal{E})$, with $\mathcal{V} = \{ 1, \ldots, N \}$ and $(i,j) \in \mathcal{E}$ if $i$ can transmit information to $j$. We denote the in-neighbors of $i$ by $\mathcal{N}_i^- = \{ j \in \mathcal{V} \; | \; (i,j)\in \mathcal{E}\}$, and the out-neighbors by $\mathcal{N}_i^+ = \{ l \in \mathcal{V} \; | \; (l, i)\in \mathcal{E}\}$.
A directed path from $i$ to $j$ of length $t$ exists if we can find a sequence of agents $i \equiv j_0,j_1, \dots, j_t \equiv j$ such that $(j_{\tau+1},j_{\tau}) \in \mathcal{E}$ for $ \tau = 0, 1, \dots , t-1$.
A digraph is \textit{strongly connected} if there exists a directed path from every agent $i$ to every agent $j$ in the network, for every $i,j \in \mathcal{V}$.
The diameter $D$ of a digraph is the longest shortest path between any two agents $i, j \in \mathcal{V}$ in the network.

\section{Problem Formulation}\label{sec:prob_form}

Given the digraph $\mathcal{G} = (\mathcal{V}, \mathcal{E})$ with $N$ agents, our goal is to solve the online optimization problem
\begin{equation}\label{eq:problem}
	x_k^* = \argmin_{x \in \R^n} \sum_{i = 1}^N f_{i,k}(x)
\end{equation}
where $f_{i,k} : \R^n \to \R$ are the time-varying local costs, each privately held by one of the agents. Since we are specifically interested in learning applications, in the following we consider local costs of the form 
\begin{equation}\label{eq:local-costs}
    f_{i,k}(x) = \frac{1}{m_i} \sum_{h = 1}^{m_i} \ell(x; d_{i,k}^h)
\end{equation}
where $\ell : \R^n \to \R$ is a loss function (\textit{e.g.} quadratic or logistic loss) and $\{ d_{i,k}^h \}_{h = 1}^{m_i}$ are the data points available to agent $i$ at time $k$ (\textit{e.g.} pairs of feature vector and label for a classification task). The problem we face then is time-varying as the local costs depend on time-varying data.

Let now $x_i \in \R^n$ be the local state of $i \in \mathcal{V}$, and let $\x = [x_1^\top, \ldots, x_N^\top]^\top \in \R^{n N}$. Then we can define the \textit{consensus set} as
$$
    \C = \{ \x \in \R^{n N} \ | \ x_i = x_j \ \forall i, j \in \mathcal{V} \},
$$
and we can rewrite~\eqref{eq:problem} as the following \textit{distributed online optimization} problem
\begin{equation}\label{eq:problem-consensus}
	\x_k^* = \argmin_{\x \in \R^{n N}} \sum_{i = 1}^N f_{i,k}(x_i) \quad \text{s.t.} \ \x \in \C. 
\end{equation} 
If the digraph is strongly connected (cf. Assumption~\ref{as:network} below), then it is clear that problems~\eqref{eq:problem} and~\eqref{eq:problem-consensus} are equivalent, and $\x_k^* = \1_N \otimes x_k^*$. 
Our goal in the following will be to design an online distributed algorithm that solves~\eqref{eq:problem-consensus} making use of a \textit{finite-time coordination routine}.

We introduce now some assumptions that will hold throughout our paper.

\begin{assumption}[Network]\label{as:network}
The digraph $\mathcal{G} = (\mathcal{V}, \mathcal{E})$ is strongly connected.
Every agent in the network knows the diameter (or an upper bound thereof) $D$ of the digraph.
\end{assumption}

\begin{assumption}[Costs]\label{as:costs}
The local cost $f_{i,k} : \R^n \to \R$ is $\lmin$-strongly convex and $\lmax$-smooth for each agent $i \in \mathcal{V}$ and time $k \in \N$.
\end{assumption}

\begin{assumption}[Bounded time-variability]\label{as:time-variability}
Problem~\eqref{eq:problem-consensus} is such that there exists $\sigma \geq 0$ such that
$$
    \norm{\x_k^* - \x_{k-1}^*} \leq \sigma, \quad \forall k \in \N.
$$
\end{assumption}

\smallskip

As mentioned above, Assumption~\ref{as:network} guarantees the equivalence of~\eqref{eq:problem} and~\eqref{eq:problem-consensus}, so that the former can be solved in a distributed fashion.
Moreover, assuming that the agents know an (upper bound to) the diameter $D$ of the digraph enables them to deploy the finite time coordination algorithm. We remark that it is possible to compute $D$ in a distributed fashion via max-consensus protocols, see \textit{e.g.} \cite{oliva_distributed_2016,deplano_distributed_2021}.

Assumptions~\ref{as:costs} and~\ref{as:time-variability} are fairly standard in online optimization, see \textit{e.g.} \cite{dallanese_optimization_2020,simonetto_time-varying_2020}, and guarantee existence and uniqueness of the optimal trajectory $\{ \x_k^* \}_{k \in \N}$ and a bounded distance between consecutive points in the trajectory, respectively.
Intuitively, Assumption~\ref{as:time-variability} implies that consecutive problems are ``similar'' so that computation performed at time $k$ can be repurposed towards an improved solution of the problem at $k+1$.

We remark that, due to the dynamic nature of the problem~\eqref{eq:problem-consensus}, convergence to the exact solution trajectory in general cannot be achieved \cite{dallanese_optimization_2020}. Rather, the proposed algorithm will converge to a neighborhood thereof, as characterized by the results presented in section~\ref{sec:convergence}.

\section{Algorithm}\label{sec:algorithm}
Conceptually, one could solve problem~\eqref{eq:problem-consensus} by applying the \textit{online projected gradient} method \cite{dixit_online_2019} characterized by
\begin{equation}\label{eq:algorithm}
	\x_{k+1} = \proj_\C\left( \x_k - \alpha \nabla f_k(\x_k) \right), \quad k \in \N,
\end{equation}
where
$
	\nabla f_k(\x) = [\nabla f_{1,k}(x_1)^\top, \ldots, \nabla f_{N,k}(x_N)^\top]^\top
$
collects the local gradients, and the projection onto the consensus space is
$
	\proj_\C(\x) = \frac{1}{N} \sum_{i = 1}^N x_i.
$
However, in order for the algorithm we design to be distributed, we need to design a distributed implementation of the projection $\proj_\C$. The idea, following \cite{2023:Rikos_Johan_IFAC}, is to implement such projection using the finite-time coordination routine outlined in Algorithm~\ref{alg:finite-time-consensus}.
Before delving into the details, we outline in the following section two challenges to which the proposed algorithm is subject: \textit{(i)} \textit{limited communication}, and \textit{(ii)} \textit{inexact gradients}.

\subsection{Challenges}\label{subsec:challenges}
Learning problems are often \textit{large-scale}, both due to the fact that the (local) costs depend on a huge number of data points, and because the model to be learned is high-dimensional. These features of the problem pose very stringent constraints on the design of solution algorithms, as discussed in the following.

\paragraph*{\textit{(i)} Limited communication}
Distributedly learning a high-dimensional model (such as deep neural networks) translates into the solution of problem~\eqref{eq:problem-consensus} with a large $n \gg 1$. But the agents need to perform consensus on the local gradient steps $\R^n \ni y_{i,k} = x_{i,k} - \alpha \nabla f_{i,k}(x_{i,k})$, demanding high communication rates.
To alleviate the communication cost, different schemes can be implemented, foremost of which is \textit{quantization} \cite{li_federated_2020,gafni_federated_2022}, which will be employed by the finite-time coordination protocol of Algorithm~\ref{alg:finite-time-consensus}. As a consequence, the agents compute only an approximate projection onto the consensus set.

\paragraph*{\textit{(ii)} Inexact gradients}
Besides the high-dimensionality of the learning problem~\eqref{eq:problem-consensus}, the local costs are also defined on high-dimensional data-sets \cite{gafni_federated_2022}, recalling~\eqref{eq:local-costs}:
$$
    f_{i,k}(x) = \frac{1}{m_i} \sum_{h = 1}^{m_i} \ell(x; d_{i,k}^h)
$$
where the number of data points available at any time $k$, $m_i$, can be very large.
Computing the exact local gradient $\nabla f_{i,k}(x)$ may have too high a computational cost, and in learning applications they are replaced by \textit{stochastic gradients}, as computed on a randomly selected sub-set of the local data \cite{li_convergence_2020}. In the following we suppose that agents compute only approximate gradients $\hat{\nabla} f_{i,k}(x)$, whose inexactness is characterized by the following assumption \cite{li_convergence_2020}.

\begin{assumption}[Stochastic gradients]\label{as:stochastic-gradients}
The agents compute the approximate local gradient $\hat{\nabla} f_{i,k}(x)$, which satisfies
$$
    \ev{\norm{\hat{\nabla} f_{i,k}(x) - \nabla f_{i,k}(x)}} \leq \tau.
$$
\end{assumption}

\subsection{Algorithm}\label{subsec:algorithm}
Accounting for the challenges discussed above, we can finally characterize the distributed implementation of the online projected gradient~\eqref{eq:algorithm} with Algorithm~\ref{alg:main-algorithm}.

\begin{algorithm}[!ht]
\caption{FTQC-DGD}
\label{alg:main-algorithm}
\begin{algorithmic}[1]
	\Require For each agent $i \in \mathcal{V}$ initialize $x_{i,0}$; choose the step-size $\alpha$.
	\For{$k = 0, 1, \ldots$ each agent $i$}
    	\CommentState{local update}
        \State receive a new local cost $f_{i,k}$
        \State apply the local (possibly inexact) gradient step
        $$
            y_{i,k} = x_{i,k} - \alpha \hat{\nabla} f_{i,k}(x_{i,k})
        $$
        \CommentState{coordination}
        \State apply the finite-time coordination Algorithm~\ref{alg:finite-time-consensus} to approximately project onto the consensus space
        $$
            \x_{k+1} = \hat{\proj}_\C(\y_k)
        $$
        with $\y_k = [y_{1,k}^\top, \ldots, y_{N,k}^\top]^\top$
    \EndFor
\end{algorithmic}
\end{algorithm}

\smallskip

The finite-time coordination routine employed to approximate the projection onto the consensus set is described in Algorithm~\ref{alg:finite-time-consensus}. 
Note here that Algorithm~\ref{alg:finite-time-consensus} is executed until it converges to the quantized average of each agent's estimate. 
Therefore, by calculating the quantized average we are able to approximate the projection onto the consensus set.
For simplicity we report it for scalar states, while in practice the agents apply Algorithm~\ref{alg:finite-time-consensus} to each component of the vectors $\{ y_{i,k} \}_{i \in \mathcal{V}}$ in parallel.

\begin{algorithm}[!ht]
\caption{Finite-time quantized coordination (FTQC)}
\label{alg:finite-time-consensus}
\begin{algorithmic}[1]
	\Require The states to be averaged $\{ x_i \}_{i \in \mathcal{V}}$, quantization level $\Delta$, and upper bound to the diameter $D$.
    \Statex{\color{blue}// initialization}
    \State each agent $i \in \mathcal{V}$ sets $y_i = 2 \lfloor x_i / \Delta \rfloor$ and $z_i = 2$
	\For{$\ell = 1, 2, \ldots$ each agent $i$}
        \Statex\hspace{1.5em}{\color{blue}// local update and transmission}
        \While{$z_i > 1$}
            \State compute $t = \lfloor y_i / z_i \rfloor$, and update $y_i = y_i - t$ and $z_i = z_i - 1$
            \State select an out-neighbor $j \in \mathcal{N}_i^+$ uniformly at random
            \State transmit $t$ to $j$, which updates $y_j = y_j + t$ and $z_j = z_j + 1$
        \EndWhile
        \Statex\hspace{1.5em}{\color{blue}// vote on termination}
        \If{$\ell \mod D = 1$}
            \State compute $\bar{m}_i = \lceil y_i / z_i \rceil$ and $\ubar{m}_i = \lfloor y_i / z_i \rfloor$, and broadcast them to the out-neighbors $j \in \mathcal{N}_i^+$
            \State receive $\bar{m}_j$, $\ubar{m}_j$ from the in-neighbors $j \in \mathcal{N}_i^-$ and set
            $
                \bar{m}_i = \max_{j \in \mathcal{N}_j^- \cup \{ i \}} \bar{m}_j, \ \ubar{m}_i = \min_{j \in \mathcal{N}_j^- \cup \{ i \}} \ubar{m}_j
            $
            \If{$\bar{m}_i - \ubar{m}_i \leq 1$}
                \State set $x_i = \Delta \ubar{m}_i$ and terminate
            \EndIf
        \EndIf
    \EndFor
\end{algorithmic}
\end{algorithm}

\smallskip

\paragraph{Algorithm~\ref{alg:main-algorithm}}
First of all, we remark that by design Algorithm~\ref{alg:main-algorithm} can be interpreted as an inexact projected gradient descent method, and be written in the abstract form
\begin{equation}\label{eq:algorithm-inexact}
	\x_{k+1} = \proj_\C\left( \x_k - \alpha \nabla f_k(\x_k) \right) + \e_k.
\end{equation}
The (random) vector $\e_k$ accounts for all sources of inexactness discussed in section~\ref{subsec:challenges}, namely the quantization error introduced by Algorithm~\ref{alg:finite-time-consensus} and the error due to inexact gradients. This interpretation of Algorithm~\ref{alg:main-algorithm} will prove to be particularly suited to studying its convergence in section~\ref{sec:convergence}.
We further remark that -- following the literature on \textit{online learning} \cite{shalev-shwartz_online_2011} -- the proposed algorithm operates in a \textit{predictive} fashion. Indeed, the local costs $\{ f_{i,k} \}_{k \in \N}$ are used to \textit{predict the solution to the problem at time $k+1$} \cite[p.~73]{dallanese_optimization_2020}.

\paragraph{Algorithm~\ref{alg:finite-time-consensus}}
It is important to notice that the finite-time coordination routine of Algorithm~\ref{alg:finite-time-consensus} is designed to be deployed over quantized communications -- indeed the agents only share integers. As a consequence, the accuracy of the projection onto $\mathcal{C}$ that it returns only depends on the quantization level $\Delta$.

Notice that the agents perform a vote on the termination of the routine every $D$ iterations, during which the agents perform a max- and a min-consensus to assess if their states have converged on an integer, thus signifying that the quantized consensus has been reached.
Indeed, Algorithm~\ref{alg:finite-time-consensus} enables the agents to compute the asymmetrically quantized average of their initial states, that is
$$
    \Bigl \lfloor \frac{\bar{x}}{\Delta} \Bigr \rfloor \Delta, \quad \bar{x} := \frac{1}{N} \sum_{i = 1}^N x_i.
$$
This is achieved with high probability, as proved in \cite{2023:Rikos_Johan_IFAC, 2022:Rikos_Johan_Split_Autom} by showing that Algorithm~\ref{alg:finite-time-consensus} can be modeled as a random walk of multiple tokes around the digraph.

\begin{lemma}[Coordination error]\label{lem:ftc-error}
Let $\x_{k+1} = \1_N \otimes x_{k+1}$, be the output of Algorithm~\ref{alg:finite-time-consensus} as applied in Algorithm~\ref{alg:main-algorithm} (line 5) to distributedly approximate $\proj_\mathcal{C}$.
Then it holds that
$$
    \norm{\proj_\mathcal{C}(\y_k) - \x_{k+1}} \leq 2 \Delta \sqrt{n N} =: \gamma.
$$
\end{lemma}
\begin{proof}
By definition $\proj_\mathcal{C}(\y_k) = \frac{1}{N} \sum_{i = 1}^N y_{i,k}$ and, since Algorithm~\ref{alg:finite-time-consensus} converges to the asymmetrically quantized average consensus, we can write
\begin{align*}
    \norm{\proj_\mathcal{C}(\y_k) - \x_{k+1}}^2 &= \sum_{i = 1}^N \norm{\frac{1}{N} \sum_{i = 1}^N y_{i,k} - x_{k+1}}^2 \\
    &\leq \sum_{i = 1}^N n (2 \Delta)^2 = n N (2\Delta)^2
\end{align*}
where the inequality holds the fact that each of the $n$ components of the vector $\frac{1}{N} \sum_{i = 1}^N y_{i,k} - x_{k+1}$ are at most $2 \Delta$ apart. The thesis follows by taking the square root.
\end{proof}

\subsection{Alternative $\proj_\C$ implementations}\label{subsec:alternatives}
In this work we have resorted to a finite-time coordination routine to (approximately) implement the projection $\proj_\C$. However, different approaches have been explored in the literature on distributed optimization and federated learning, and it is instructive to briefly discuss them.

\paragraph{Average consensus}
Assuming that the graph is undirected\footnote{If this is not the case, a \textit{push-sum} or \textit{ratio} consensus algorithm can be employed instead \cite{tsianos_push-sum_2012}.}, and letting $\mathbold{W}$ be a doubly stochastic matrix for the network, one may choose to approximate $\proj_\C$ with a number $t \in \N$ of consensus steps for each time $k$, yielding the \textit{online version of Near-DGD} \cite{berahas_balancing_2019}
\begin{equation}\label{eq:near-dgd}
    \x_{k+1} = \mathbold{W}^t \left( \x_k - \alpha \nabla f_k(\x_k) \right), \quad k \in \N.
\end{equation}
As characterized in \cite[Lemma~V.2]{berahas_balancing_2019}, $\mathbold{W}^t$ is indeed an inexact projection onto $\C$, whose error decreases as $t$ increases (as a consequence of the well known convergence of the average consensus).

\paragraph{Gradient tracking}
Rather than (approximately) projecting local gradient steps onto the consensus set, the extensively studied approach of \textit{gradient tracking} can be used instead, see \textit{e.g.} \cite{xin_decentralized_2020}. This paradigm performs a projection onto $\C$ asymptotically, by introducing additional variables and using average consensus.
In the context of this paper, it is interesting to notice that \textit{online versions of gradient tracking methods may actually be less performant than ``DGD-style'' methods} such as~\eqref{eq:algorithm} and~\eqref{eq:near-dgd}, see \cite{yuan_can_2020,bastianello_distributed_2021}.

\paragraph{Centralized aggregation}
In \textit{federated learning} a widely used assumption is that a central server is available, with which all agents can communicate, and which has the task of aggregating the local updates (\textit{e.g.} via weighted averaging as in FedAvg \cite{mcmahan_communication-efficient_2017}). When all the agents provide a local gradient at time $k$, the server can then perform an exact projection onto $\C$. However, this approach relies on the central server, which becomes a single point of failure. In this paper we are interested in fully distributed algorithms that rely on peer-to-peer communications.

\section{Convergence Analysis}\label{sec:convergence}
In section~\ref{subsec:challenges} we introduced the abstract characterization of Algorithm~\ref{alg:main-algorithm} as the inexact online projected gradient~\eqref{eq:algorithm-inexact}. It is straightforward to see that, by adding and subtracting the right vectors, we can indeed write~\eqref{eq:algorithm-inexact} with $\e_k = \e_k^p + \e_k^g$ where
\begin{align*}
    \e_k^p &:= \hat{\proj}_\C(\x_k - \alpha \hat{\nabla} f_k(\x_k)) - \proj_\C(\x_k - \alpha \hat{\nabla} f_k(\x_k)) \\
    \e_k^g &:= \proj_\C(\x_k - \alpha \hat{\nabla} f_k(\x_k)) - \proj_\C(\x_k - \alpha \nabla f_k(\x_k))
\end{align*}
with $\e_k^p$ modeling the inexact projection onto the consensus set $\C$ due to the use of the FTQC scheme in Algorithm~\ref{alg:finite-time-consensus}, and $\e_k^g$ the use of approximate local gradients. The following result analyzes the convergence of~\eqref{eq:algorithm-inexact} in mean, accounting for the effects of time-variability of the costs and of $\e_k$.

\begin{proposition}[Convergence in mean]\label{pr:mean-convergence}
Consider the online distributed problem~\eqref{eq:problem-consensus}, and suppose that Assumptions~\ref{as:network}--\ref{as:stochastic-gradients} hold. Let $\{ \x_k \}_{k \in \N}$ be the trajectory generated by Algorithm~\ref{alg:main-algorithm}, then it holds that for all $k > 0$:
$$
    \ev{\norm{\x_k - \x_k^*}} \leq \zeta^k \norm{\x_0 - \x_0^*} + \left( \sigma + \gamma + \alpha \tau \right) \frac{1 - \zeta^{k+1}}{1 - \zeta}
$$
where $\zeta = \max\{ |1 - \alpha \lmin|, |1 - \alpha \lmax| \} \in (0, 1)$, and $\sigma$, $\gamma$, $\tau$ are defined in Assumption~\ref{as:time-variability}, Lemma~\ref{lem:ftc-error}, and Assumption~\ref{as:stochastic-gradients}, respectively.
\end{proposition}
\begin{proof}
It is straightforward to see that we can write~\eqref{eq:algorithm-inexact}
\begin{align*}
    \x_{k+1} - \x_{k+1}^* &= \proj_\C\left( \x_k - \alpha \nabla f_k(\x_k) \right) - \x_k^* + \\ &+ \left( \x_k^* - \x_{k+1}^* + \e_k^p + \e_k^g \right)
\end{align*}
where we introduced an additional error ($\x_k^* - \x_{k+1}^*$) due to the fact that the algorithm is predictive and approximates $\x_{k+1}^*$ using the costs at time $k$. Taking the norm, redefining $\e_k = \x_k^* - \x_{k+1}^* + \e_k^p + \e_k^g$, and using the triangle inequality we write
\begin{align}
    &\norm{\x_{k+1} - \x_{k+1}^*} \leq \norm{\proj_\C\left( \x_k - \alpha \nabla f_k(\x_k) \right) - \x_k^*} + \norm{\e_k} \nonumber \\
    &\qquad \overset{(i)}{\leq} \zeta \norm{\x_k - \x_k^*} + \norm{\e_k}, \label{eq:tracking-error}
\end{align}
with $\zeta = \max\{ |1 - \alpha \lmin|, |1 - \alpha \lmax| \} \in (0, 1)$, where (i) follows since by Assumption~\ref{as:costs} problem~\eqref{eq:problem-consensus} is $\lmin$-strongly convex and $\lmax$-smooth \cite{taylor_exact_2018} and the fact that $\x_k^* = \proj_\C(\x_k^* - \alpha \nabla f_k(\x_k^*))$. Taking the expected value we have $\ev{\norm{\x_{k+1} - \x_{k+1}^*}} \leq \zeta \ev{\norm{\x_k - \x_k^*}} + \ev{\norm{\e_k}}$; we bound now the three terms in $\ev{\norm{\e_k}}$ separately.

By Assumption~\ref{as:time-variability} we know that $\norm{\x_k^* - \x_{k+1}^*} \leq \sigma$, and by Lemma~\ref{lem:ftc-error} we know that $\ev{\norm{\e_k^p}} \leq \gamma$. Finally, by Assumption~\ref{as:stochastic-gradients} we get
\begin{align*}
    \ev{\e_k^g} &\overset{(i)}{\leq} \ev{\norm{\x_k - \alpha \hat{\nabla} f_k(\x_k) - (\x_k - \alpha \nabla f_k(\x_k))}} = \\
    &= \alpha \ev{\norm{\hat{\nabla} f_k(\x_k) - \nabla f_k(\x_k)}} \leq \alpha \tau
\end{align*}
where (i) holds since the projection is $1$-Lipschitz continuous \cite[Proposition 4.16]{bauschke_convex_2017}. Iterating~\eqref{eq:tracking-error} and using the geometric series yields the thesis.
\end{proof}

\smallskip

Notice that taking the limit for $k \to \infty$ from Proposition~\ref{pr:mean-convergence} we get
\begin{equation}\label{eq:asymptotic-error}
    \lim_{k \to \infty} \ev{\norm{\x_k - \x_k^*}} \leq \frac{\sigma + \gamma + \alpha \tau}{1 - \zeta}.
\end{equation}
We can then see how the different sources of ``inexactness'' -- time-variability of the costs, inexact coordination due to the use of Algorithm~\ref{alg:finite-time-consensus}, approximated gradients -- play a role in the bound for the asymptotic tracking error $\lim_{k \to \infty} \ev{\norm{\x_k - \x_k^*}}$.
We remark in particular that the bound $\tau$ is multiplied by the step-size $\alpha$, which implies that the smaller we choose it, the less the effect of inexact gradients; however, the smaller $\alpha$ is the larger $\zeta$, implying a trade-off\footnote{One option to choose the step-size then is to find the value of $\alpha$ that minimizes the asymptotic error bound given in \eqref{eq:asymptotic-error}.}.

\begin{remark}
We remark that the results can be easily extended to allow for a different bound $\norm{\x_k^* - \x_{k-1}^*} \leq \sigma_k$ at each time $k \in \N$.
\end{remark}

\section{Numerical Results}\label{sec:simulations}



In this section we consider the \textit{(online) logistic regression} problem characterized by problem~\eqref{eq:problem} with the local cost functions
\begin{equation}\label{eq:logistic-cost}
    f_{i,k}(x) = \sum_{h = 1}^{m_i} \log\left( 1 + \exp\left( - b_{i,k}^h a_{i,k}^h x \right) \right) + \frac{\epsilon}{2} \norm{x}^2,
\end{equation}
with agent $i \in \mathcal{V}$ at time $k \in \N$ storing the private data $\{ (a_{i,k}^h, b_{i,k}^h) \}_{h = 1}^{m_i}$, with $a_{i,k}^h \in \R^{1 \times n}$ and $b_{i,k}^h \in \{ -1, 1 \}$, and $\epsilon = 5$. The network is composed of $N = 10$ agents, the number of unknowns (including the intercept) is $n = 16$, and each agent stores $m_i = 20$ data points.

Following the discussion of section~\ref{subsec:alternatives}, we will compare FTQC-DGD with Near-DGD \cite{berahas_balancing_2019} and the distributed gradient tracking (DGT) method \cite{yuan_can_2020}. All three methods use quantized communications, and are given the same communication budget to ensure a fair comparison\footnote{In practice, this means that Near-DGD and DGT perform the same number of communication rounds as required in FTQC-DGD. For DGT these communications are split among the two steps characterizing the algorithm.}.

\subsection{Quantization}\label{subsec:numerical-quantization}
We start evaluating the performance of the three algorithms by applying them to a static logistic regression problem (the data in~\eqref{eq:logistic-cost} do not change) and for different quantization levels $\Delta$.

First of all, in Figure~\ref{fig:quantization} we report the error trajectory $\{ \norm{\x_k - \x^*} \}_{k \in \N}$ generated by the three algorithms when the communications are quantized with $\Delta = 0.01$. As previously observed in \textit{e.g.} \cite{pu_robust_2020}, gradient tracking methods are not robust to quantization, which leads DGT to diverge.
FTQC-DGD and Near-DGD both converge to a neighborhood of the optimal solution, in accordance with the results of section~\ref{sec:convergence}, with the neighborhood reached by the former being smaller.

\begin{figure}[!ht]
    \centering
    \includegraphics[scale=0.6]{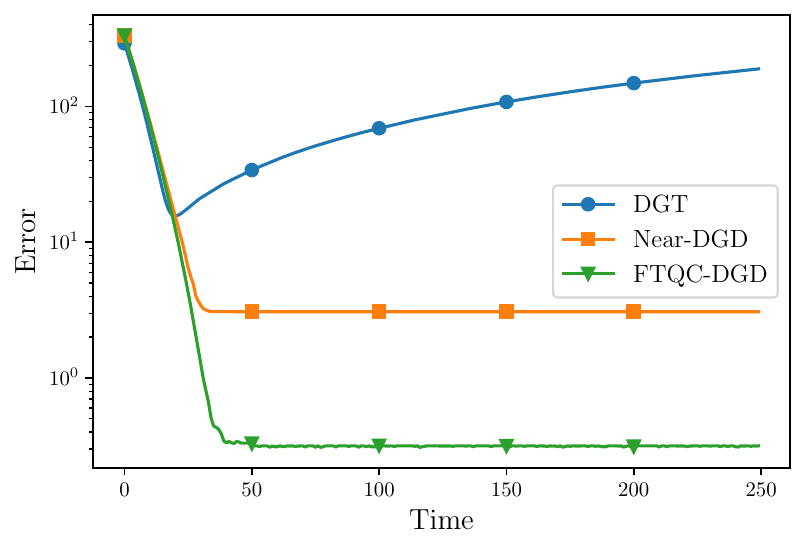}
    \caption{Comparison of the proposed FTQC-DGD with Near-DGD and DGT, using quantized communications with $\Delta = 0.01$.}
    \label{fig:quantization}
\end{figure}

In Table~\ref{tab:quantization} we report the asymptotic error\footnote{Computed as the maximum error in the last $4/5$ of the simulation.} for Near-DGD and FTQC-DGD, varying the quantization level $\Delta$. As we can see, the proposed algorithm consistently achieves smaller errors, with the error being one order of magnitude smaller for values $\Delta \leq 0.1$.

\begin{table}[!ht]
\begin{center}
\caption{Asymptotic error achieved by Near-DGD and FTQC-DGD for different quantization levels.}
\label{tab:quantization}
\begin{tabular}{ccc}
    Quantization            & Near-DGD                  & FTQC-DGD                  \\
    \hline
    $\Delta = 10^{-5}$      & $4.49 \times 10^{-3}$     & $3.79 \times 10^{-4}$     \\
    $\Delta = 10^{-4}$      & $3.51 \times 10^{-2}$     & $3.65 \times 10^{-3}$     \\
    $\Delta = 10^{-3}$      & $2.18 \times 10^{-1}$     & $3.43 \times 10^{-2}$     \\
    $\Delta = 10^{-2}$      & $3.07$                    & $3.16 \times 10^{-1}$     \\
    $\Delta = 10^{-1}$      & $41.91$                   & $3.72$                    \\
    $\Delta = 1$            & $236.58$                  & $124.20$                  \\
    $\Delta = 10$           & $566.37$                  & $414.18$                  \\
    \hline
\end{tabular}
\end{center}
\end{table}

\subsection{Stochastic gradients}
In this section we apply the proposed algorithm to the static logistic regression problem, with quantization level fixed to $\Delta = 10^{-4}$, and are interested in evaluating its performance when the agents can only access inexact local gradient. As discussed in section~\ref{subsec:challenges}, when the local cost~\eqref{eq:logistic-cost} depends on a large number of data points ($m_i \gg 1$), then evaluating a local gradient may be too computationally expensive. Instead, an agent $i$ each time selects a subset of its data, $\mathcal{H} \subset \{ 1, \ldots, m_i \}$, uniformly at random, with $b := |\mathcal{H}| < m_i$ being the batch size. The agent can then approximately evaluate the gradient by computing \cite{xin_decentralized_2020}
$
    \hat{\nabla} f_{i,k}(x) = \sum_{h \in \mathcal{H}} \nabla \log\left( 1 + \exp\left( - b_{i,k}^h a_{i,k}^h x \right) \right) + \epsilon x.
$

In Table~\ref{tab:batch-size} we report the asymptotic error attained by the proposed algorithm for different batch sizes. Clearly, the smaller the batch size the worse $\hat{\nabla} f_{i,k}(x)$ approximates the true gradient, compounding the quantization error with a second source of inexactness.

\begin{table}[!ht]
\begin{center}
\caption{Asymptotic error achieved by FTQC-DGD for different batch sizes.}
\label{tab:batch-size}
\begin{tabular}{cc}
    Batch size              & As. error                 \\
    \hline
    $b = m_i / 2$           & $2.37$                    \\
    $b = 11$                & $2.05$                    \\
    $b = 12$                & $1.82$                    \\
    $b = 13$                & $1.68$                    \\
    $b = 14$                & $1.38$                    \\
    $b = 15$                & $1.19$                    \\
    $b = 16$                & $9.78 \times 10^{-1}$     \\
    $b = 17$                & $7.51 \times 10^{-1}$     \\
    $b = 18$                & $6.05 \times 10^{-1}$     \\
    $b = 19$                & $3.66 \times 10^{-1}$     \\
    $b = m_i$               & $3.45 \times 10^{-3}$     \\
    \hline
\end{tabular}
\end{center}
\end{table}

\subsection{Online optimization}
We conclude this section by evaluating the performance of the proposed algorithm when applied to an online logistic regression problem, in which the local cost functions change every $100$ iterations.

In Figure~\ref{fig:online} we report the tracking error $\{ \norm{\x_k - \x_k^*} \}_{k \in \N}$ of FTQC-DGD for different values of the quantization level $\Delta$. As we can see, while the problem does not change, the algorithm converges to a neighborhood of the optimal solution, whose radius depends on the quantization error (as discussed in section~\ref{subsec:numerical-quantization}). Each time the problem changes, the tracking error spikes, due to the fact that the new optimal solution is different from the previous one and the algorithm undergoes a new transient. As characterized in the results of section~\ref{sec:convergence}, the upper bound to the tracking error indeed depends on the maximum distance between consecutive optima.

\begin{figure}[!ht]
    \centering
    \includegraphics[scale=0.6]{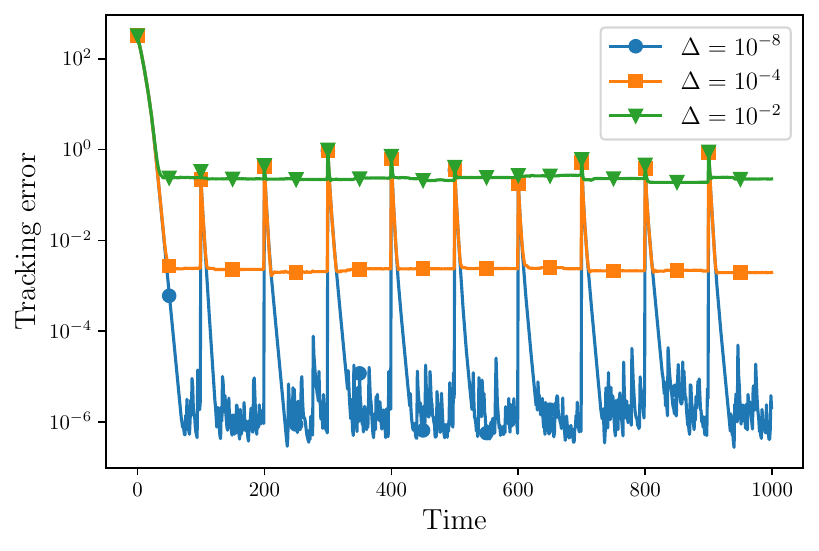}
    \caption{Tracking error of FTQC-DGD applied to an online logistic regression problem, with different quantization levels.}
    \label{fig:online}
\end{figure}

\section{Conclusions and Future Directions}\label{sec:conclusions}
In this paper we proposed a novel gradient-based algorithm for distributed learning over fully decentralized, directed networks. In particular, the algorithm substitutes a fusion center (employed in federated learning) with a finite-time coordination protocol. We analyzed the performance of the algorithm accounting for the use of quantized communications and stochastic gradients, as well as the impact of time-varying costs.
Future work will look into comparing the use of finite-time coordination with different quantization schemes, and applying the proposed approach to more general problems such as composite and constrained optimization problems.




\bibliographystyle{IEEEtran}
\bibliography{references}

\end{document}